\documentclass[sigconf]{acmart}

\usepackage{booktabs, multicol, multirow}
\usepackage{algorithm}
\usepackage{algpseudocode}
\usepackage{subcaption}
\AtBeginDocument{%
  }

\setcopyright{acmlicensed}
\copyrightyear{2018}
\acmYear{2018}
\acmDOI{XXXXXXX.XXXXXXX}

\acmConference[Conference acronym 'XX]{Make sure to enter the correct
  conference title from your rights confirmation emai}{June 03--05,
  2018}{Woodstock, NY}
\acmISBN{978-1-4503-XXXX-X/18/06}

\author{Wenjie Yang}
\email{yangwj22@m.fudan.edu.cn}
\affiliation{%
  \institution{Fudan University}
  \city{Shanghai}
  \country{China}
}

\author{Shengzhong Zhang}
\email{szzhang17@fdu.edu.cn}
\affiliation{%
  \institution{Fudan University}
  \city{Shanghai}
  \country{China}
}

\author{Jiaxing Guo}
\email{jxguo24@m.fudan.edu.cn}
\affiliation{%
  \institution{Fudan University}
  \city{Shanghai}
  \country{China}
}

\author{Zengfeng Huang}
\email{huangzf@fudan.edu.cn}
\affiliation{%
  \institution{Fudan University}
  \city{Shanghai}
  \country{China}
}
\begin{document}


\title{Your Graph Recommender is Provably a Single-view Graph Contrastive Learning}

\renewcommand{\shortauthors}{Trovato et al.}

\newtheorem{assumption}[theorem]{Assumption}
\newtheorem{remark}[theorem]{Remark}

\newcommand{\adj}{\mathbf{A}}
\newcommand{\feat}{\mathbf{X}}
\newcommand{\sigmoid}{\sigma_\mathrm{s}}
\newcommand{\neighbor}{\mathcal{N}}
\begin{abstract} 
    Graph recommender (GR) is a type of graph neural network (GNNs) encoder that is customized for extracting information from the user-item interaction graph. Due to its strong performance on the recommendation task, GR has gained significant attention recently. Graph contrastive learning (GCL) is also a popular research direction that aims to learn, often unsupervised, GNNs with certain contrastive objectives. As a general graph representation learning method, GCLs have been widely adopted with the supervised recommendation loss for joint training of GRs. Despite the intersection of GR and GCL research, theoretical understanding of the relationship between the two fields is surprisingly sparse. This vacancy inevitably leads to inefficient scientific research.

    In this paper, we aim to bridge the gap between the field of GR and GCL from the perspective of encoders and loss functions. With mild assumptions, we theoretically show an astonishing fact that \emph{graph recommender is equivalent to a commonly-used single-view graph contrastive model}. Specifically, we find that (1) the classic encoder in GR is essentially a linear graph convolutional network with one-hot inputs, and (2) the loss function in GR is well bounded by a single-view GCL loss with certain hyperparameters. The first observation enables us to explain crucial designs of GR models, e.g., the removal of self-loop and nonlinearity. And the second finding can easily prompt many cross-field research directions. We empirically show a remarkable result that the recommendation loss and the GCL loss can be used interchangeably. The fact that we can train GR models solely with the GCL loss is particularly insightful, since before this work, GCLs were typically viewed as unsupervised methods that need fine-tuning. We also discuss some potential future works inspired by our theory.
\end{abstract}




\maketitle

\section{Introduction}

Graph data are ubiquitous in real-world applications, with recommendation systems\cite{he2020lightgcn, wang2019ngcf, wu2021sgl, yu2022simgcl} serving as one of the most representative examples. In recommendation systems, users and items can be considered as nodes in a graph, and their interactions form the edges between these nodes. By modeling the use-item interactions in this way, the task of recommending user interests transforms into a prediction problem on graph.
Several early works are proposed to tackle this collaborative filtering task, such as SVD++ \cite{koren2008factorization} and Neural Attentive Item Similarity (NAIS) \cite{he2018nais}. In recent years, a type of powerful graph encoder called graph neural network has been applied to graph recommendation \cite{he2020lightgcn, wu2021sgl}. GNNs enable each node to gather information from its neighboring nodes through the message-passing operation, and then utilize the aggregated representations to learn downstream tasks. However, node features are often unavailable in the context of collaborative filtering. To address this issue and cater to personalized recommendations, LightGCN \cite{he2020lightgcn} is proposed, and soon becomes the most commonly used baseline model. Instead of transformation matrices, LightGCN first initializes a learnable embedding for each node. It then propagates these embeddings on the graph and computes the loss. Despite its effectiveness, we are surprised to find that the details of LightGCN lack sufficient theoretical explanations. For instance, \emph{what are the differences between LightGCN and other GNN encoders, e.g., \cite{kipf2016semi}? Why do the modifications made in LightGCN lead to its effectiveness in graph recommendation tasks?} In this paper, we show that these research questions can be rigorously analysed by a simple bridge between encoders.

Recent GR models also include graph contrastive learning to boost their performance \cite{wu2021sgl, yu2022simgcl, yu2023xsimgcl}. The primary idea of GCLs is to learn (often unsupervised) graph encoders by minimizing the distance between positive pairs and maximizing the distance between negative pairs, and it has been a popular research direction even outside the GR field \cite{zhang2020sce, zhu2021coles, zhu2020grace}. Despite this overlap between GCL and GR, research conducted on these two fields is rather parallel. For example, the scalability issue is studied on GCL by some works \cite{zheng2022ggd, zhang2023structcomp}, and is also studied on GR by some other works \cite{zhang2024linear, chen2024macro}. In this paper, we theoretically show some astonishing results that bridge the gap between GCL and GR. With these findings, we can easily identify redundancies in these two fields and avoid inefficient research efforts. Specifically, we show:

\textbf{\emph{Graph recommender is provably a single-view graph contrastive learning model.}}

This finding answers lots of research questions and brings novel insights. We can immediately know that the recommendation loss and the single-view GCL loss can be used interchangeably, which yields a new GCL method and a new way to train GR models. The fact that we can train GR models solely with the GCL loss is particularly interesting since GCLs were typically viewed as unsupervised methods in previous works. We anticipate that our results will lead to numerous novel research works. Research on GCL, such as scalability and negative sample mining, can be transferred to GR, and vice versa. Our contributions can be summarized as follows:

\begin{itemize}

    \item We carefully review the GR and GCL fields and present the application of GCL in GR. We include the training pipeline of both directions in our paper, making it self-contained and suitable as reading material for beginners.

    \item We demonstrate the equivalence between GR and GCL from the perspectives of both encoders and loss functions. This discovery not only provides theoretical explanations for the effectiveness of existing models but also deepens our understanding of both directions. Most importantly, based on our theory, research in GCL and GR can be mutually inspired by each other, and new directions can be prompted.

    \item We conduct extensive experiments to demonstrate the interchangeability of GCL and GR losses. The fact that training graph recommender models solely use GCL loss is remarkable because previous research often consider GCL as an unsupervised loss function that must be used in conjunction with downstream tasks. In the field of recommendation, GCL has been regarded as unsuitable even for pre-training. However, our theory easily explains why using GCL loss alone is sufficient for training graph recommender models.
\end{itemize}

\noindent\textbf{Outline.} In Section \ref{sec:rel}, we discuss the related works. To ensure the self-contained nature of this paper, we provide separate introductions to the algorithmic processes of GR and GCL for authors unfamiliar with these fields in Section \ref{sec:pre}. Section \ref{sec:theory} presents the main conclusions of the article, highlighting the establishment of connections between encoders and the proof of equivalence between losses. In Section \ref{sec:exp}, we demonstrate experimental results to further validate our theoretical findings and provide examples of new research directions.

\section{Related Work} \label{sec:rel}

\noindent\textbf{Graph Contrastive Learning.} Graph contrastive learning is a powerful tool for graph representation learning. There are a two lines of research on GCLs, namely single-view GCLs and multi-view GCLs. Similar to contrastive learning in the computer vision field, multi-view GCLs (e.g., \cite{zhu2020grace, zhu2021gca, zhang2021ccassg, zheng2022ggd, chen2023polygcl}) perform data augmentations to generate corrupted views of the target graph and consider synthetic nodes originating from the same node as positive pairs, while treating the remaining nodes as negative samples. GRACE \cite{zhu2020grace} is a representative multi-views GCL, it performs edge dropping and feature masking data augmentation, then choose negative pair from both inter-view and intra-view. GCA \cite{zhu2021gca} is an improved version of GRACE, which performs adaptive augmentation instead of the handcrafted ones. There are also works like CCA-SSG \cite{zhang2021ccassg}, which takes the perspective of canonical correlation analysis. Scalability issue has also been studied, GGD \cite{zheng2022ggd} use group discrimination instead of individual level. There are also more sophisticated model, like PolyGCL \cite{chen2023polygcl} that learns spectral polynomial filters.

With the graph information, one can also do single-view GCLs (e.g., \cite{hamilton2017sage, zhang2020sce, zhu2021coles, wang2022spgcl}), which consider adjacent nodes as positive pairs and unconnected nodes as negative samples. SCE \cite{zhang2020sce}, motivated by the classical sparsest cut problem, only uses negative samples from unconnected nodes. It is the first GCL work that does not rely on data augmentation. COLES \cite{zhu2021coles} further improves SCE by treating connected nodes as positive samples. Single-view GCLs are typically more efficient and less expensive to train compared to multi-views GCLs.

The original GCLs follow the pre-train then fine-tune paradigm, they first learn representations with the unsupervised contrastive learning, than fit a linear classifier for the downstream task. Many research works have been conducted on GCLs, including hard negative mining \cite{xia2022progcl} and scalability \cite{zhang2023structcomp}. With the connection established between GCL and GR, these studies can be easily applied to graph recommender systems.

\noindent\textbf{Graph Recommender.} Collaborative filtering (CF) is one of the most classical methods for recommender systems. The primary idea of CF is to filter out items that a user might like with user-item interactions. In recent years, GNNs have been incorporated into CF-based recommender system. These graph recommender is able to learn complex structural pattern, and they are more expressive than traditional CF. For example, an early work NGCF \cite{wang2019ngcf} performs message-passing on the user-item bipartite graph and trains embeddings with recommendation loss. Afterward, LightGCN \cite{he2020lightgcn} removes the non-linear activation functions and feature transformations in GR, simplifying the model while achieving promising performance. Despite LightGCN becoming the mainstream graph recommender, we are astonished to find that previous works have obtained minimal theoretical understanding regarding these simplifications.

Graph contrastive learning is also used in graph recommender. To the best of our knowledge, they all use multi-views GCLs. SGL \cite{wu2021sgl} is the first work to incorporate the InfoNCE \cite{gutmann2010infonce} loss into GR. Following SGL, SimGCL \cite{yu2023xsimgcl} is proposed to discard graph augmentation, achieving a more efficient model without sacrificing performance. Unlike traditional GCLs, the contrastive learning loss in these papers is learnt jointly with the downstream recommendation loss. There is also a work called CPTPP \cite{yang2024cptpp} that studies the new prompt-tuning framework. However, to the best of our knowledge, our paper is the first work that demonstrate the feasibility of training a viable GR model solely using GCL. ContraRec \cite{wang2023contrarec} is a recent work that also try to build connection between the recommendation loss and the contrastive learning loss. Unlike our work, ContraRec focus on the sequential recommendation, which results in a compromising contrastive loss that is not used in application. Furthermore, the experiments in ContraRec are conducted used an extended version of the recommendation loss, while our paper shows that training GR with off-the-shelf single-view GCLs is possible. We believe that the latter brings more insight, since GCLs were viewed as unsupervised losses in the past.
\section{Preliminaries} \label{sec:pre}

We consider the undirected graph $G=(\adj, \feat)$, where $\adj \in \{ 0, 1\}^{n \times n}$ is the adjacency matrix, and $\feat \in \mathbb{R}^{n\times f}$ is the feature matrix. For GR, the bipartite adjacency matrix can also be formulated as
\begin{equation}
    \adj = \begin{pmatrix}
\mathbf{0} & \mathbf{R} \\
\mathbf{R}^T & \mathbf{0}
\end{pmatrix},
\end{equation}
where $\mathbf{R}$ is the user-item interaction matrix. The feature matrix $\feat$ is not available in GR. The set of vertices and edges is represented as $V$ and $E$. We also denote $n=|V|$ and $m=|E|$. The degree of the $i$-th node is $d_i$, and the diagonal degree matrix is $\mathbf{D}$. The symmetrically normalized matrix is $\tilde{\adj}=\mathbf{D}^{-1/2}\adj \mathbf{D}^{-1/2}$, and the symmetrically normalized matrix with self-loop is $\hat{\adj}=(\mathbf{D}+\mathbf{I})^{-1/2}(\adj+\mathbf{I}) (\mathbf{D}+\mathbf{I})^{-1/2}$. The neighbor set of a node $x$ is denoted as $\neighbor_x$. In both GR and GCL, the negative samples are randomly selected from all unconnected nodes, we denote the negative sample set as $\neighbor_x^-$ and $K=|\neighbor_x^-|$. The Laplacian matrix of $\adj$ is $\mathbf{L}=\mathbf{D}-\adj$. We also have $\mathbf{L}^-$ as the randomly generated Laplacian matrices capturing the negative sampling.

\subsection{GCN}

Graph convolutional network (GCN) \cite{kipf2016semi} is the most common GNN encoder. Given the node representations $\mathbf{H}^{(l)}$ of the $l$-th layer, the next layer representations are computed as follow:
\begin{equation} \label{eq:gcn_prop}
    \mathbf{H}^{(l+1)}=\sigma(\hat{\adj}\mathbf{H}^{(l)}\mathbf{W}^{(l)}),
\end{equation}
where $\mathbf{W}^{(l)}$ is a learnable parameter matrix and $\sigma(\cdot)$ is the activation function (e.g., ReLU).  GCNs consist of multiple convolution layers, with the initial representations $\mathbf{H}^{(0)}=\feat$.

\subsection{LightGCN}

LightGCN \cite{he2020lightgcn} is a widely used GR model. Unlike GCN, LightGCN trains user and item embeddings instead of transformation matrix. Let the $0$-th embedding matrix be $\mathbf{E}^{(0)}\in \mathbb{R}^{(n_u+n_i)\times d}$, where $n_u$ and $n_i$ are the numbers of users and items, respectively. For GR, we also denote $n=n_u+n_i$. The propagation rule of LightGCN is formulated as follow:
\begin{equation} \label{eq:lightgcn_prop}
    \mathbf{E}^{(l+1)}=\tilde{\adj}\mathbf{E}^{(l)}.
\end{equation}
The final embedding matrix is a weighted average of each layer, i.e.,
\begin{equation} \label{eq:lightgcn_final}
    \mathbf{E}=\sum_{l=0}^{L-1} \alpha_{l} \tilde{\adj}^{l}\mathbf{E}^{(0)},
\end{equation}
where $L$ is the number of layers. For simplicity, we set $\alpha_l=\frac{1}{L+1}$ throughout this paper.

In LightGCN, the score between a pair of user and item is the inner product of their final embedding:
\begin{equation}
    \hat{y}_{ui}=e_u^T e_i,
\end{equation}
where $e_u$ and $e_i$ are embeddings of $u$ and $i$ from $\mathbf{E}$. The model can be trained using the bayesian personalized ranking (BPR) loss \cite{rendle2012bpr}:
\begin{equation} \label{eq:bpr}
    \mathcal{L}_{\mathrm{BPR}}=-\sum_{u=1}^{n_u}\sum_{i\in \neighbor_u}\sum_{j\in \neighbor_u^-}\ln\sigmoid(\hat{y}_{ui}-\hat{y}_{uj})+\lambda ||\mathbf{E}^{(0)}||^2,
\end{equation}
where $\sigmoid(\cdot)$ is the sigmoid function and $\lambda$ controls the $L_2$ regulation.

In order to make our paper self-contained, we provide the pipeline of LightGCN for readers who are unfamiliar with GR. Please refer to Algorithm \ref{alg:gr} for details.
\begin{algorithm}
\caption{The pipeline of LightGCN}\label{alg:gr}
\begin{algorithmic}[1]
\Require The training adjacency matrix $\adj$, the number of hidden dimensions $d$, the regulation coefficient $\lambda$.
\State Initialize the embedding lookup table $\mathbf{E}^{(0)}\in \mathbb{R}^{n\times d}$.
\While{not converge} \Comment{Training}
\State Propagate embeddings $\mathbf{E}^{(0)}$ with (\ref{eq:lightgcn_prop}).
\State Get final embeddings $\mathbf{E}$ with (\ref{eq:lightgcn_final}).
\State Compute the loss $\mathcal{L}_{\mathrm{BPR}}$ with (\ref{eq:bpr}).
\State Back propagation and update the embeddings $\mathbf{E}^{(0)}$.
\EndWhile
\For{$(u, i)$ in the test set} \Comment{Testing}
\State Compute the score $\hat{y}_{ui}=e^T_u e_i$.
\EndFor
\State Sort the test set by $\hat{y}$ and compute the metrics.
\end{algorithmic}
\end{algorithm}

\subsection{(Single-view) GCLs}

GCL is a class of unsupervised graph representation
learning method. It aims to learn the embeddings of the graph by distinguishing between similar and dissimilar nodes. In this paper, we mainly focus on single-view GCLs, in which a GCN model, denoted as $\mathrm{GCN}(\adj, \feat)$, is trained on the original graph $G$ without data augmentation. We take COLES \cite{zhu2021coles} as an example of single-view GCL. Its objective is as follow:
\begin{equation} \label{eq:coles}
    \begin{split}
        \mathcal{L}_{\mathrm{COLES}}&=\mathrm{Tr}(\mathbf{E}^{T}\mathbf{L}\mathbf{E})-\beta \mathrm{Tr}(\mathbf{E}^{T}\mathbf{L}^-\mathbf{E})\\
        &=\mathcal{L}_{\mathrm{COLES}}^+-\beta\mathcal{L}_{\mathrm{COLES}}^-,
    \end{split}
\end{equation}
where $\beta$ is a hyperparameter. For convenience, we divide the loss function into two parts.

Most of GCL researches use node classification as the downstream task (e.g., \cite{zhang2020sce, zhu2021coles}). They first learn representations with unsupervised objectives, then train a separate classifier with the downstream task loss. We also provide a pipeline of ordinary GCLs in Algorithm \ref{alg:gcl}.

\begin{algorithm}
\caption{The pipeline of ordinary GCLs}\label{alg:gcl}
\begin{algorithmic}[1]
\Require The adjacency matrix $\adj$, the feature matrix $\feat$, the number of layers $L$, the negative coefficient $\beta$, labels of certain downstream task $y$.
\State Pre-compute the Laplacian matrix $\mathbf{L}$ and sample the negative Laplacian matrix $\mathbf{L}^-$.
\State Initialize the learnable parameters $\mathbf{W}^{(l)}, \forall l \in [L]$.
\While{not converge} \Comment{Pre-training}
\State Get the final embedding $\mathbf{E}$ with (\ref{eq:gcn_prop}).
\State Compute the unsupervised loss $\mathcal{L}_{\mathrm{COLES}}$ with (\ref{eq:coles}).
\State Back propagation and update the parameters $\mathbf{W}^{(l)}$.
\EndWhile
\State Get and freeze the embedding $\mathbf{E}$.
\State Get train/test split. \Comment{Downstream Task}
\State Fit a linear (or MLP) classifier on the training set with lables $y$ and embeddings $\mathbf{E}$.
\State Predict labels for the test set with the linear classifier.
\end{algorithmic}
\end{algorithm}

In GR, there are also models that utilize GCLs (e.g., \cite{wu2021sgl, yu2023xsimgcl}). However, these models are jointly learned with the BPR and contrastive loss, instead of the pre-training then fine-tuning paradigm. Specifically, they simply add certain contrastive loss into Line 5 of Algorithm \ref{alg:gr}.
\section{Theoretical Analysis} \label{sec:theory}

In this section, we theoretically bridge the gap between GCL and GR. We start with a mild assumption.

\begin{assumption}
Throughout this paper, we assume the embedding of each node is normalized, i.e., $||e_x||=1, \forall x \in V$.
\end{assumption}

The assumption of normalized embedding is widely used in deep learning theory (e.g., \cite{huang2022ssl}). In GCL, it is also easy to normalized the embeddings for downstream tasks. However, naive adoption of normalization may hinder the performance of recommender system. Fortunately, a recent work \cite{chung4leveraging} proposes a debiasing regulation to address this issue. In this paper, we conduct GR experiments with normalized embeddings and this regulation to ensure the consistency between theory and practice.

\subsection{The equivalence between encoders} \label{sec:encoder}

We first investigate the duality between GCN and LightGCN encoder. The results help us gain a better understanding of why LightGCN works well in GR. A well-known fact is that personalized embeddings can be viewed as a linear layer with one-hot inputs. Therefore, we have the following proposition.

\begin{proposition} \label{pro:encoder}
Let users and items have one-hot features, LightGCN is a GCN without non-linearity and self-loop.
\end{proposition}

\begin{proof}
    For GR, the feature matrix $\feat$ is not available. If we let $\mathbf{H}^{(0)}=\feat=\mathbf{I}$, and the learnable parameter matrix $\mathbf{W}^{(0)}=\mathbf{E}^{(0)}$. The last layer embeddings of LightGCN are simply:
    \begin{equation}
        \mathbf{E}^{(L)}=\tilde{\adj}^L XW^{(0)}.
    \end{equation}
    This is equivalent to the embeddings of GCN without non-linearity and self-loop. 
\end{proof}

It is not difficult to understand why LightGCN works without self-loop. In GR, the average of embeddings in each layer is adopted instead of only the last layer. Thus, each node is still able to acquire its own information, only with different weights. The original paper of LightGCN \cite{he2020lightgcn} shows this under a compromising form where the symmetric normalization is discarded. Alternatively, we can use the relative influence in \cite{chen2022redundancy, lampert2023selfloop} to show this.

\begin{lemma}[The relative influence in GNN \cite{lampert2023selfloop}] \label{lemma:influence}
 If the GNN passes messages along all $k$-length walks from $u$ to $v$ with equal probability, then the relative influence of input feature $e^{(0)}_u$ on node output $e^{(k)}_v$ is on the average
 \begin{equation}
     \mathbb{E}\left(\frac{\partial e^{(k)}_v/\partial e^{(0)}_u}{\sum_{u'\in V}\partial e^{(k)}_v/\partial e^{(0)}_{u'}}\right)=\frac{\adj_{uv}^k}{\sum_{u'\in V} \adj_{u'v}^k},
 \end{equation}
 where $\adj_{uv}^k$ computes the number of walks of length $k$ from node $u$ to $v$.
\end{lemma}
With lemma \ref{lemma:influence}, we can compare the relative influence of each nodes and its own output between the self-loop added GNN and the average embedding GNN. Specifically, for GNNs with self-loop, the relative influence is $\frac{\tilde{\adj}_{uu}^L}{\sum_{u'\in V} \tilde{\adj}_{u'u}^L}$. For GNNs with average embeddings as the output, like LightGCN, we have $\sum_{l=0}^{L-1} \frac{\adj_{uu}^l}{\sum_{u'\in V} \adj_{u'u}^l}$. Consider a 2-layer GNN, the former becomes $\frac{1+d_u}{1+3d_u+d_u^{(2)}}$, and the later becomes $1 + \frac{d_u}{d_u+d_u^{(2)}}$, where $d_u^{(2)}$ is the number of 2-hop neighbors of node $u$. We observe that the self-loop added GNN receive less information from its own input than the average embedding GNN. We also compare these two variants empirically and present the results in Section \ref{sec:add_exp}.

However, comprehending the removal of non-linearity poses greater challenges as it plays a crucial role in deep learning. With Proposition \ref{pro:encoder}, we know that this is due to the fact that one-hot features do not contain rich information, and non-linearity in GNN is only useful when the node attributes are far more informative than the graph structure \cite{wei2022nonlinear}. Formally, \cite{wei2022nonlinear} considers the contextual stochastic block model (CSBM) \cite{deshpande2018csbm}. In CSBM, labels are first sampled from Rademacher distribution. Based on the node labels, node features are sampled from $\mathbb{P}_{1}$ and $\mathbb{P}_{-1}$. We further set $\mathbb{P}_{1}=\mathcal{N}(\mu, 1/d)$ and $\mathbb{P}_{-1}=\mathcal{N}(\nu, 1/d)$ for some $\mu, \nu \in \mathbb{R}^{d}$. For the node pair $(u,v)$ subject to $y_u=y_v$, the probability that an edge exists is $p$. If $y_u\neq y_v$, the probability is $q$. This model can be translated into the GR task by viewing each label as whether the item is positive or negative for a particular user $u$, and the edges between items as "item-user-item" meta-paths. Denote the non-linear propagation model as $\mathcal{P}$, and the optimal linear counterpart as $\mathcal{P}^l$, we can compute the signal-to-noise ratio (SNR) $\rho_r$ and $\rho_l$.

\begin{theorem}[The SNR of nonlinear and linear propagation model \cite{wei2022nonlinear}] \label{the:snr} Assume the structure information $\mathcal{S}(p, q)=(p-q)^2/(p+q)$ is moderate, i.e., $\mathcal{S}(p,q)=\omega_n(\frac{(\ln n)^2}{n})$ and $\frac{S(p,q)}{|p-q|}\not\to_n 1$. And $\sqrt{d}||\mu-\nu||=\omega_n(\sqrt{\ln n/S(p,q)n})$, we have:
\begin{itemize}
            \item \textbf{I. Limited Attributed Information}: When $\sqrt{d}\|\mu - \nu\| = \mathcal{O}_n(1)$, 
                \begin{align}\label{eq:lim-regime}
                    \rho_r = \Theta_n(\rho_l),
                \end{align}
                Further, if $\sqrt{d}\|\mu - \nu\|_2 = o_n(|\log(p/q)|)$, $\rho_r / \rho_l \to_n 1$;
            \item \textbf{II. Sufficient Attributed Information}: When $\sqrt{d}\|\mu - \nu\|_2  = \omega_n(1)$ and $\sqrt{d}\|\mu - \nu\|_2  =o_n(\ln n)$, we have
            \begin{equation}
                \begin{aligned}
                \rho_r &= \omega_n(\min\{\exp(d\|\mu - \nu\|_2^2/3), nS(p, q) d^{-1}\|\mu - \nu \|_2^{-2}\} \cdot \rho_l)\\ &=\omega_n(\rho_l).
                \end{aligned}
            \end{equation}
            
        \end{itemize}
    
\end{theorem}

Theorem \ref{the:snr} shows that when attributed information is limited, the SNR of regular GNNs with nonlinear propagation has the same order as its optimal linear counterpart. LightGCN in GR tasks clearly falls into this category. Since orders are not correlated to whether the item is positive or negative, the expectation of node inputs are the same: $\mu=\nu=\frac{1}{n}\mathbf{1}$, and $\sqrt{d}||u - v||=0=O_n(1)$.

\subsection{The equivalence between loss functions}

\begin{figure*}[t] 
    \centering
    \begin{subfigure}[b]{0.32\textwidth}
        \centering
        \includegraphics[width=\textwidth]{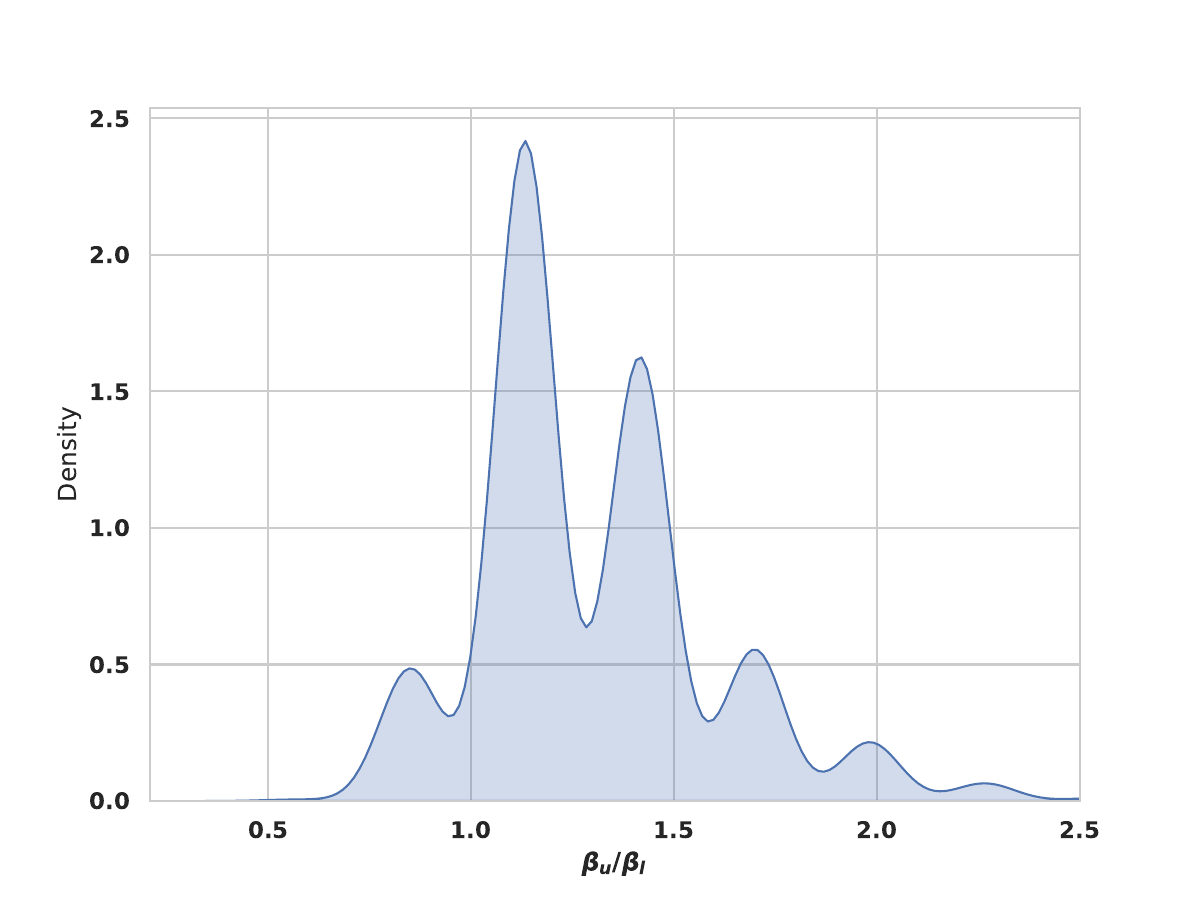}
        \caption{Yelp2018}
    \end{subfigure}
    \hfill
    \begin{subfigure}[b]{0.32\textwidth}
        \centering
        \includegraphics[width=\textwidth]{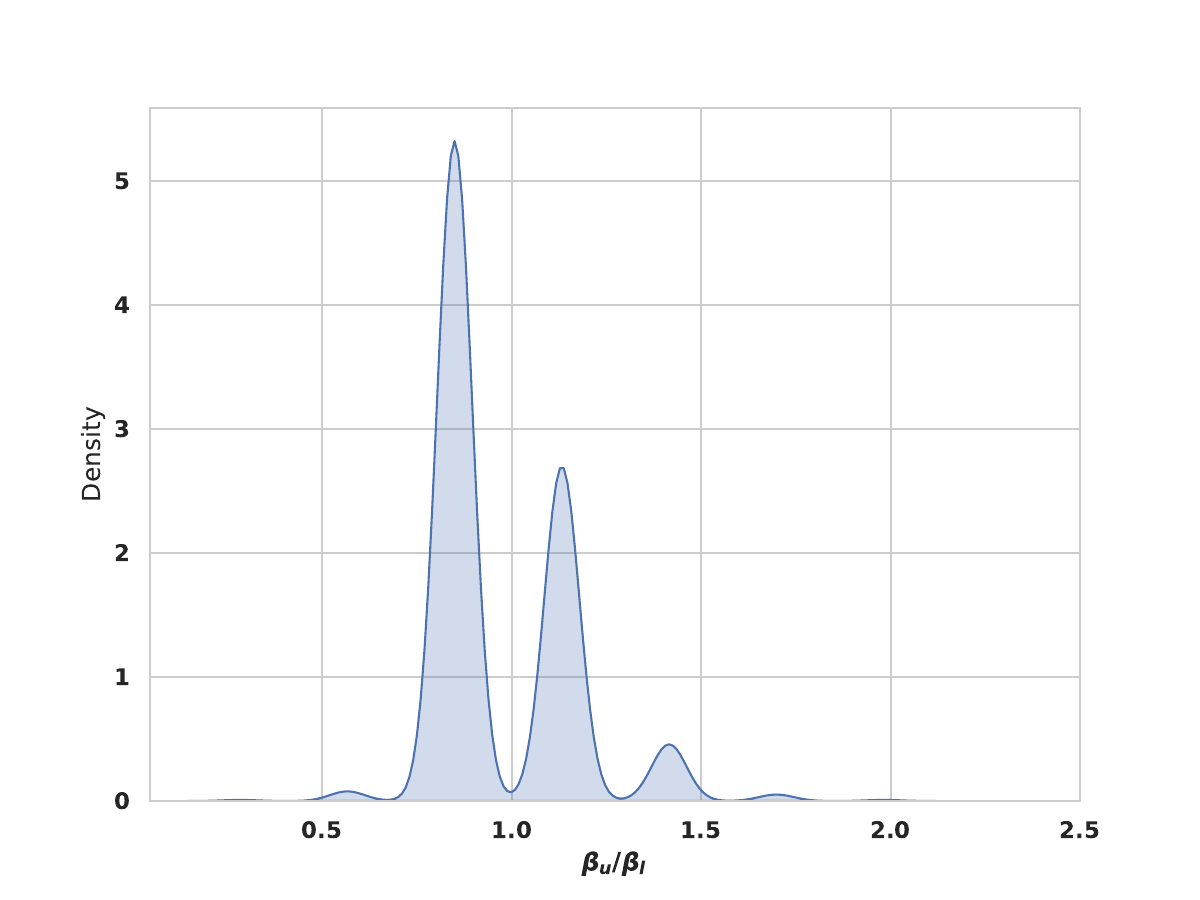}
        \caption{Amazon-Kindle}
    \end{subfigure}
    \hfill
    \begin{subfigure}[b]{0.32\textwidth}
        \centering
        \includegraphics[width=\textwidth]{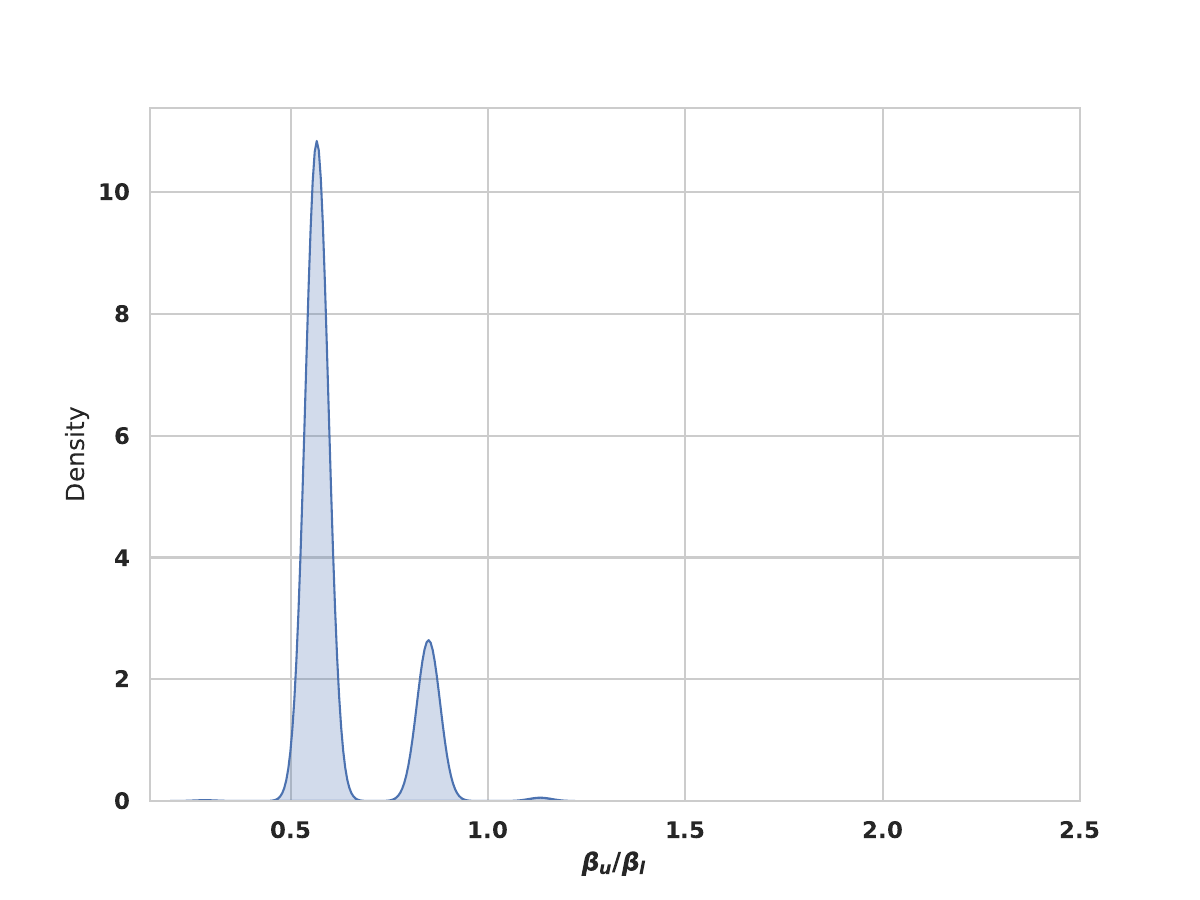}
        \caption{Alibaba-iFashion}
    \end{subfigure}
    \caption{The ratio of negative coefficient $\beta_u/\beta_l$ on three real-world datasets.}
    \label{fig:beta_ratio}
\end{figure*}

We then investigate the equivalence between the loss functions, establishing a bridge between GR and GCL, which can simultaneously facilitate the advance of both fields. We first present the following lemma, the proof is deferred to Appendix \ref{app:proofs}.

\begin{lemma} \label{lem:trace}
    The trace of the quadratic form of the Laplacian matrix is equal to the smoothness on the graph, i.e.,

    \begin{equation}
        \mathrm{Tr}(\mathbf{E}^T \mathbf{L} \mathbf{E})=\sum_{u=1}^{n_u}\sum_{i\in \mathcal{N}_u}||e_u-e_i||^2.
    \end{equation}
\end{lemma}

Lemma \ref{lem:trace} makes $\mathcal{L}_{\mathrm{COLES}}$ easier to deal with, and can be found in standard textbooks, e.g.,  \cite{hamilton2020graphbook}. We then present our result in the following main theorem, the proof is deferred to Appendix \ref{app:proofs}.

\begin{theorem} \label{the:main}
With the normalized embeddings, we have the following equivalence:
\begin{equation}
    \begin{split}
        \frac{K}{2}\mathcal{L}_{\mathrm{COLES}}^+  -\frac{d_{\mathrm{min}}}{2}\mathcal{L}_{\mathrm{COLES}}^- +d_{\mathrm{min}}Kn_u - mK \leq \mathcal{L}_{\mathrm{BPR}} \\
        \leq \frac{K}{2}\mathcal{L}_{\mathrm{COLES}}^+ - \frac{d_{\mathrm{max}}}{4}\ln(\frac{2e^2}{e^2+1})\mathcal{L}_{\mathrm{COLES}}^- + \mathrm{constant},
    \end{split}
\end{equation}
where $\mathrm{constant}=d_{\mathrm{max}}K n_u\ln(\frac{2e^3+2e}{e^2+1}) - mK$.
\end{theorem}

In GR researches, the number of negative samples is often set as one, i.e., $K=1$. Therefore, the hyperparameter that control the effect of the negative terms $\beta$ is $\beta_{l}=d_{\mathrm{min}}$ for the lower bound and $\beta_{u}=\frac{d_{\mathrm{max}}}{2}\ln(\frac{2e^2}{e^2+1})$ for the upper bound. If $\beta_u/\beta_l$ is close enough to $1$, we can conclude that the equivalence of GR and GCL loss is good. We empirically show this on three real-world datasets, Yelp2018, Amazon-Kindle and Alibaba-iFashion. We use 1,000 batches to get the kernel density estimation of $\beta_u/\beta_l$, and plot the distribution in Figure \ref{fig:beta_ratio}. We observe that $\beta_u/\beta_l$ is roughly concentrated around $1$, which indicates the corresponding negative coefficient of the upper bound and the lower bound are similar.

It can be  concluded from Theorem \ref{the:main} that, the performance of GNNs on the recommendation task, namely $\mathcal{L}_{\mathrm{BPR}}$, can be well bounded by certain weighted versions of the contrastive loss $\mathcal{L}_{\mathrm{COLES}}$. In other words, training the model using GR loss and GCL loss is essentially equivalent. With this information, we can facilitate knowledge exchange between the two fields. The research on single-view GCL can be applied to GR. Similarly, the findings in GR research can also assist in enhancing GCL. We give some research examples in the next section.
\section{Experiments} \label{sec:exp}

\begin{table}[]
    \centering
    \caption{Summary of the datasets used in node classification.}
    \begin{tabular}{c|cccc}
    \toprule
      \textbf{Dataset}   & \textbf{\#Nodes} &  \textbf{\#Edges}&\textbf{\#Features} & \textbf{\#Classes} \\
      \midrule
       Cora  & 2,708 & 5,429 &1,433 &  7\\
       CiteSeer & 
3,327 & 4,732 & 3,703 & 6 \\
       PubMed & 19,717 & 44,338 & 500 & 3 \\
       Computers & 13,752 & 491,722 & 767 & 10 \\
       Photo&7,650&238,163 & 745 & 8 \\
       Co-CS & 18,333 & 163,788 & 6,805 & 15 \\
       Co-Physics & 34,493 & 495,924 & 8,415 & 5\\
       \bottomrule
    \end{tabular}
    \label{tab:stat1}
\end{table}

In this section, we validate through experiments that the loss functions of GCL and GR can be used interchangeably. These experiment can be divided into two categories, namely \emph{GR-inspired GCL research} and \emph{GCL-inspired GR research}. We believe that our findings can prompt a broader range of more sophisticated researches, which we leave for future work.

\subsection{GR-inspired GCL research}

\begin{table*}[h]
    \centering
    \caption{Evaluation of BPR loss on unsupervised single-view GCL. We use classification accuracy as the metric. All experiments are repeated 10 times. We \textbf{bold} the best-performing method and \underline{underline} the second-best method for each dataset. We cannot get the results of SP-GCL on CiteSeer since it does not support graphs with isolated nodes.}
    \begin{tabular}{lccccccc}
    \toprule
        \textbf{Methods} & \textbf{Cora} &\textbf{CiteSeer} & \textbf{PubMed} & \textbf{Computers} & \textbf{Photo}&\textbf{Co-CS} &\textbf{Co-Physics}\\
        \midrule
        GraphSAGE & 77.8 $\pm$ 1.3 & 68.2 $\pm$ 1.5 & 75.5 $\pm$ 0.9 & 68.3 $\pm$ 0.7 & 73.2 $\pm$ 0.7 & 90.7 
        $\pm$ 0.4 & 94.3 $\pm$ 0.1\\
        SCE& 80.9 $\pm$ 0.8 & 70.0 $\pm$ 0.7 & 76.3 $\pm$ 0.7 & 84.4 $\pm$ 0.8 & 89.5 $\pm$ 0.6 & 91.6 $\pm$ 0.2 & \underline{95.5} $\pm$ 0.1\\
        COLES & \textbf{81.8} $\pm$ 0.4 & \textbf{70.7} $\pm$ 1.0 & 75.6 $\pm$ 1.2 & 85.9 $\pm$ 0.4 & \textbf{92.1} $\pm$ 0.4 & 92.0 $\pm$ 0.2 & \textbf{95.6} $\pm$ 0.1\\
        SP-GCL & 81.5 $\pm$ 0.5 & - & \textbf{78.9} $\pm$ 1.0 & \textbf{87.7} $\pm$ 0.5 &  \underline{91.9} $\pm$ 0.5 & \textbf{92.4} $\pm$ 0.3 & 94.8 $\pm$ 0.1\\
        \midrule
        BPR & \underline{81.6} $\pm$ 0.6 & \underline{70.6} $\pm$ 0.8 & \underline{78.7} $\pm$ 0.8 & \underline{86.1} $\pm$ 0.5 & 91.7 $\pm$ 0.4 & \underline{92.2} $\pm$ 0.3 & 95.2 $\pm$ 0.1\\
    \bottomrule
    \end{tabular}
    
    \label{tab:bpr_vs_svgcl}
\end{table*}

\begin{table*}[h]
    \centering
    \caption{Evaluation of joint training between COLES and various multi-views GCLs. We use classification accuracy as the metric. All experiments are reported 10 times.}
    \begin{tabular}{lccccccc}
    \toprule
        \textbf{Methods} & \textbf{Cora} &\textbf{CiteSeer} & \textbf{PubMed} & \textbf{Computers} & \textbf{Photo}&\textbf{Co-CS} &\textbf{Co-Physics}\\
        \midrule
        GRACE & 83.5 $\pm$ 0.3 & 70.1 $\pm$ 0.9 & 80.2 $\pm$ 0.7 & 85.4 $\pm$ 0.4 & 90.2 $\pm$ 0.4 & 92.1 $\pm$ 0.3 & 95.5 $\pm$ 0.1\\
        GCA & 83.3 $\pm$ 0.5 & 68.7 $\pm$ 0.6 & 81.2 $\pm$ 0.6 & 88.9 $\pm$ 0.5 & 92.3 $\pm$ 0.3 & 92.2 $\pm$ 0.1 & 95.5 $\pm$ 0.1\\
        CCA-SSG & 82.1 $\pm$ 0.4 & 70.5 $\pm$ 0.9 & 81.8 $\pm$ 0.7 & 88.4 $\pm$ 0.6 & 92.5 $\pm$ 0.2 & 92.5 $\pm$ 0.4 & 95.4 $\pm$ 0.2\\
        GGD & 81.7 $\pm$ 0.5 & 69.6 $\pm$ 0.7 & 80.4 $\pm$ 0.4 & 85.5 $\pm$ 0.8 & 92.1 $\pm$ 0.4 & 92.4 $\pm$ 0.4 & 95.2 $\pm$ 0.1\\
        PolyGCL & 80.1 $\pm$ 0.4 & 70.2 $\pm$ 0.5 & 79.2 $\pm$ 0.7 & 82.8 $\pm$ 0.6 & 89.2 $\pm$ 0.5 & 91.3 $\pm$ 0.2 & 95.5 $\pm$ 0.2\\
        \midrule
        COLES + GRACE & 83.6 $\pm$ 0.5 & 70.8 $\pm$ 0.6 & 80.7 $\pm$ 0.8 & 85.5 $\pm$ 0.4 & 90.6 $\pm$ 0.4 & 92.2 $\pm$ 0.2 & 95.5 $\pm$ 0.1\\
        COLES + GCA & 83.7 $\pm$ 0.4& 70.6 $\pm$ 0.7 & 81.5 $\pm$ 0.5 & 89.1 $\pm$ 0.3 & 92.4 $\pm$ 0.4 & 92.5 $\pm$ 0.2 & 95.6 $\pm$ 0.2\\
        COLES + CCA-SSG & 82.6 $\pm$ 0.5 & 72.0 $\pm$ 0.4 & 81.9 $\pm$ 0.5 & 88.4 $\pm$ 0.5 & 92.7 $\pm$ 0.3 & 92.9 $\pm$ 0.5 & 95.5 $\pm$ 0.1\\
        COLES + GGD & 82.3 $\pm$ 0.7 & 70.9 $\pm$ 0.8 & 80.4 $\pm$ 0.6 & 86.2 $\pm$ 0.3 & 92.5 $\pm$ 0.2 & 92.5 $\pm$ 0.3 & 95.6 $\pm$ 0.1\\
        COLES + PolyGCL & 81.8 $\pm$ 0.6 & 71.0 $\pm$ 0.7 & 80.2 $\pm$ 0.7 & 86.2 $\pm$ 0.5 & 92.5 $\pm$ 0.4 & 92.7 $\pm$ 0.4 & 95.6 $\pm$ 0.1\\
    \bottomrule
    \end{tabular}
    
    \label{tab:mvgcl}
\end{table*}

We start with the GR-inspired GCL research. A natural question is: \emph{Can we use the BPR loss to train GNNs in an unsupervised manner?} To investigate this question, we conduct unsupervised graph representation learning on seven datasets and evaluate the effectiveness of the models using node classification, the results are reported in Table \ref{tab:bpr_vs_svgcl}. We use Cora, CiteSeer and PubMed \cite{yang2016cora}, Amazon-Computers and Amazon-Photo, Coauthor-CS and Coauthor-Physics \cite{shchur2018amazon}. The data statitics are in Table \ref{tab:stat1}. For baselines, we include common single-view GCLs: 

\begin{itemize}
    \item Unsupervised GraphSAGE \cite{hamilton2017sage}: An early work that samples positive pairs by random walk.
    \item SCE \cite{zhang2020sce}: A  sparsest cut inspired, negative samples only GCL method.
    \item COLES \cite{zhu2021coles}: A contrastive Laplacian eigenmaps method with both positive and negative pairs.
    \item SP-GCL \cite{wang2022spgcl}: A single-view GCL method that works on both homophilic and heterophilic graphs.
\end{itemize}

For all datasets, we first train a GNN with 512 hidden dimensions, then fit a linear classifier and report its accuracy. We use the bpr loss $\mathcal{L}_{\mathrm{BPR}}$ to replace the COLES loss $\mathcal{L}_{\mathrm{COLES}}$, and dub the model "BPR". In GCL, this can be done by treating connected nodes as the positive pair, and sampling one negative sample for each anchor node. With the contrastive pairs, we simply train the GNN with Equation (\ref{eq:bpr}) instead of Equation (\ref{eq:coles}). From Table \ref{tab:bpr_vs_svgcl}, we observe that BPR works well as a single-view GCL, despite not achieving state-of-the-art performance. This result further supports the validity of our findings.

Previous works have reported that the performance of GR can be improved if it is jointly trained with multi-view GCLs. Since we have obtained the equivalence between the bpr loss and single-view GCLs, it is intriguing to ask: \emph{Will joint training with multi-view GCLs benefits single-view GCL?} We conduct experiments on the previously mentioned seven datasets. We modify the loss function of COLES $\mathcal{L}_{\mathrm{COLES}}$ to the following function:
\begin{equation}
    \mathcal{L}_{\mathrm{COLES + MVGCL}}=\mathcal{L}_{\mathrm{COLES}}+\gamma\mathcal{L}_{\mathrm{MVGCL}},
\end{equation}
where $\mathcal{L}_{\mathrm{MVGCL}}$ is the loss function of a certain multi-view GCL and $\gamma$ is a hyperparameter. We choose five representative methods to be jointly trained with COLES:
\begin{itemize}
    \item GRACE \cite{zhu2020grace}: An early multi-view GCL method that uses an extended version of the InfoNCE loss. The edge removing graph augmentation technique is proposed in this method.
    \item GCA \cite{zhu2021gca}: An improved version of GRACE that uses adaptive augmentation instead of the static one.
    \item CCA-SSG \cite{zhang2021ccassg}: It optimizes a feature-level objective inspired by the canonical correlation analysis. Methods like CCA-SSG are not viewed as contrastive models in some paper, but they are generally equivalent according to \cite{garrido2023duality}.
    \item GGD \cite{zheng2022ggd}: A scalable GCL method with group discrimination.
    \item PolyGCL \cite{chen2023polygcl}: A recent GCL method with polynomial filters.
\end{itemize}

We observe that, joint training of COLES and multi-view GCLs indeed improves the performance on most of datasets. It could be attribute to the regulation effects brought by graph augmentation. On many datasets, this combination also exceeds the original multi-view GCLs. This is also intuitive since the single-view GCL includes explicit structure information that helps models to identify more positive samples.

\subsection{GCL-inspired GR research} \label{sec:gcl4gr}

\begin{table}[]
    \centering
    \caption{Summary of the datasets used in graph recommender.}
    \begin{tabular}{c|ccc}
    \toprule
        \textbf{Dataset} & \textbf{\#Users} & \textbf{\#Items} &
        \textbf{\#Edges}\\
    \midrule
     Yelp2018    & 31,668 & 38,048 & 1,561,406\\
     Amazon-Kindle & 138,333 &  98,572 & 1,909,965\\
     Alibaba-iFashion & 300,000 & 81,614 & 1,607,813\\
     \bottomrule
    \end{tabular}
    \label{tab:stat2}
\end{table}

\begin{table*}[h]
    \centering
    \caption{Evaluation of GR models that trained solely with the single-view GCL loss. We use Recall@20 and NDCG@20 as the metrics.}
    \begin{tabular}{lcccccc}
    \toprule
     \multirow{2}{*}{\textbf{Methods}}  & \multicolumn{2}{c}{\textbf{Yelp2018}} & \multicolumn{2}{c}{\textbf{Amazon-Kindle}} & \multicolumn{2}{c}{\textbf{Alibaba-iFashion}}  \\
     \cmidrule(lr){2-3}  \cmidrule(lr){4-5} \cmidrule{6-7}
     
         & \textbf{Recall@20} & \textbf{NDCG@20} & \textbf{Recall@20} & \textbf{NDCG@20} & \textbf{Recall@20} & \textbf{NDCG@20}\\
         \midrule
         LightGCN & 0.0639 & 0.0525 & 0.2057 & 0.1315 & 0.0955 & 0.0461	\\
         LightGCN$_\mathrm{COLES}$ & 0.0623 & 0.0507 & 0.1996 & 0.1302 & 0.0973 & 0.0458\\
         \midrule
         SGL & 0.0675 & 0.0555 & 0.2069 & 0.1328 & 0.1032 & 0.0498\\
         SGL$_\mathrm{COLES}$ & 0.0651 & 0.0548 & 0.2035 & 0.1322 & 0.1024 & 0.0489 \\
         \midrule
         SimGCL & 0.0721 & 0.0601 & 0.2104 & 0.1374 & 0.1151 & 0.0567 \\
         SimGCL$_\mathrm{COLES}$ & 0.0703 & 0.0592 & 0.2075 & 0.1355 & 0.1097 & 0.0546\\
         \midrule
         XSimGCL & 0.0723 & 0.0604 & 0.2147 & 0.1415 & 0.1196 & 0.0586\\
         XSimGCL$_\mathrm{COLES}$ & 0.0710 & 0.0594 & 0.2080 & 0.1378 & 0.1105 & 0.0566\\
    \bottomrule
    \end{tabular}
    \label{tab:coles_replace_bpr}
\end{table*}

 We then provide several examples of GCL-inspired GR research. Similar to the previous section, we begin by validating the following question: \emph{Can we train a graph recommender solely using single-view GCL loss?} To do so, we want to replace the bpr loss $\mathcal{L}_{\mathrm{bpr}}$ with a single-view GCL loss $\mathcal{L}_{\mathrm{COLES}}$ in Algorithm \ref{alg:gr}. In addition, we also need to normalize the final embedding $\mathbf{E}$ to match with our assumption. However, the naive normalization is known to be harmful to recommender system \cite{chung4leveraging}, i.e., force embeddings to reside on the surface of a unit hypersphere will cause dramatic performance degrade. Fortunately, there is also a regularization term to address this problem in \cite{chung4leveraging}. Given the user set $\mathcal{U}$ and the item set $\mathcal{I}$, they first calculate a sum of pairwise Gaussian potentials for each set homogeneously:
 \begin{equation}\label{eq:hom}
     \mathcal{L}_{\mathrm{hom}}=\sum_{x\in \mathcal{U}} \sum_{y\in \mathcal{U}} e^{-t||e_x - e_y||^2}+\sum_{x\in \mathcal{I}} \sum_{y\in \mathcal{I}} e^{-t||e_x - e_y||^2},
 \end{equation}
 where $t$ is a hyperparameter. Then, they also use the heterogeneous loss between positive and negative samples:
 \begin{equation}\label{eq:het}
     \mathcal{L}_{\mathrm{het}}=\sum_{u=1}^{n_u}\sum_{i\in \neighbor_u}\sum_{j\in \neighbor_u^-} e^{-t||e_i - e_j||^2}.
 \end{equation}
The previous study \cite{chung4leveraging} have demonstrated that the combination of Equation (\ref{eq:hom}) and (\ref{eq:het}) effectively mitigates the bias introduced by normalization. With this regulation term, we finally arrive our training objective:
\begin{equation} \label{eq:gr_coles}
    \mathcal{L}_{\mathrm{GR}_{\mathrm{COLES}}}=\mathcal{L}_{\mathrm{COLES}} + \mathcal{L}_{\mathrm{hom}} + \mathcal{L}_{\mathrm{het}}.
\end{equation}
Previous works regarded contrastive learning loss like $\mathcal{L}_{\mathrm{COLES}}$ as unsupervised, making it difficult to understand why Equation (\ref{eq:gr_coles}) can directly train a good GR model without fine-tuning. However, this becomes very natural with Theorem \ref{the:main}, because we know that $\mathcal{L}_{\mathrm{COLES}}$ and $\mathcal{L}_{\mathrm{BPR}}$ are equivalent. We replace $\mathcal{L}_{BPR}$ with $\mathcal{L}_{\mathrm{GR}_{\mathrm{COLES}}}$, and conducted experiments on multiple GR models:

\begin{itemize}
    \item LightGCN \cite{he2020lightgcn}: An early work of graph recommender and the most commonly-used backbone model. It initializes a learnable embedding lookup table and propagates it using the message-passing operation. Then the embeddings are optimized with the BPR loss.
    \item SGL \cite{wu2021sgl}: The GCL loss is adopted for jointly training with the BPR loss. We change the BPR loss to the single-view GCL loss and keep the original contrastive module, the model is dubbed SGL$_{\mathrm{COLES}}$.
    \item SimGCL \cite{yu2022simgcl}: It performs embedding augmentation instead of graph augmentation for efficiency. We change the BPR loss to the single-view GCL loss and keep the original contrastive module, the model is dubbed SimGCL$_{\mathrm{COLES}}$.
    \item XSimGCL \cite{yu2023xsimgcl}: It performs contrastive learning across different layer embeddings for further simplification. We change the BPR loss to the single-view GCL loss and keep the original contrastive module, the model is dubbed XSimGCL$_{\mathrm{COLES}}$.
\end{itemize}

We use three real-world datasets: Yelp2018 \cite{he2020lightgcn}, Amazon-kindle \cite{wu2021sgl} and Alibaba-iFashion \cite{wu2021sgl}. We report Recall@20 and NDCG@20 in Table \ref{tab:coles_replace_bpr}, these two metrics are commonly used in previous works \cite{he2020lightgcn, wu2021sgl, yu2023xsimgcl}. It is clear that the single-view GCL loss works well in GR. This finding is novel because previous research believed that GCL could only be trained in conjunction with the BPR loss, and even contrastive pre-training was considered challenging \cite{yang2024cptpp}. However, our study demonstrates for the first time that GCL can be used as a standalone downstream task in graph recommender systems. This finding also enables many research studies on GCL to directly apply to GR, such as scalability \cite{zhang2023structcomp} and negative sample mining \cite{xia2022progcl}. We leave these aspects for future work.

\subsection{Additional Experiments} \label{sec:add_exp}

\textbf{Self-loop and average.} As mentioned in Section \ref{sec:encoder}, the LightGCN encoder uses the average embeddings of all layers instead of just the final embedding. Despite the fact that it does not add self-loop into the adjacency matrix, nodes in LightGCN actually receives more information from the inputs than the vanilla GCN. Here, we empirically investigate these two options of architecture. We add self-loop into the adjacency matrix and take only the final embeddings as the output, this model is dubbed as LightGCN$_{\mathrm{Selfloop}}$. We compare the original average embedding LightGCN with the modified version. The results are presented in Table \ref{tab:self_loop}. We observe that two encoders produce similar performance. This is predictable since they capture same information just with different weights. Additionally, LightGCN$_{\mathrm{Selfloop}}$ performs better on the Yelp2018 dataset, while the original LightGCN performs better on Amazon-Kindle and Alibaba-iFashion. We believe the reason for this is that Yelp2018 is with a significantly larger density (0.13\%, while Amazon-Kindle has 0.014\% and Alibaba-iFashion 0.007\%). The rich neighborhood information is better captured with the LightGCN$_{\mathrm{Selfloop}}$ according to our analysis in Section \ref{sec:encoder}, since each node has less relative influence on itself in this encoder. 

\begin{table}[h]
    \centering
    \caption{Evaluation of LightGCN$_{\mathrm{Selfloop}}$ and the original LightGCN. We use Recall@20 and NDCG@20 as the metric. The better results are bolded.}
    \begin{tabular}{cccc}
    \toprule
     \textbf{Dataset}    &  \textbf{Metric} & \textbf{LightGCN$_{\mathrm{Selfloop}}$} & \textbf{LightGCN}\\
     \midrule
     \multirow{2}{*}{\textbf{Yelp2018}}    & \textbf{Recall@20} & \textbf{0.0642} & 0.0639\\
     & \textbf{NDCG@20} &\textbf{0.0528} & 0.0525\\
    \multirow{2}{*}{\textbf{Amazon-Kindle}}& \textbf{Recall@20} &0.2023 & \textbf{0.2057}\\
    & \textbf{NDCG@20} & \textbf{0.1317} & 0.1315\\
    \multirow{2}{*}{\textbf{Alibaba-iFashion}}& \textbf{Recall@20} & 0.0940 &\textbf{0.0955}\\
    & \textbf{NDCG@20} & 0.0450 & \textbf{0.0461}\\
    \bottomrule
    \end{tabular}
    \label{tab:self_loop}
\end{table}

\noindent\textbf{Sensitivity analysis.} We conduct sensitivity analysis on the negative coefficient $\beta$ used in GR experiments. We present the result on Yelp2018 in Figure \ref{fig:hyper}. We observe that $\beta$ does not significantly affect the performance of GR in reasonable range.

\begin{figure}
    \centering
    \includegraphics[width=0.45\textwidth]{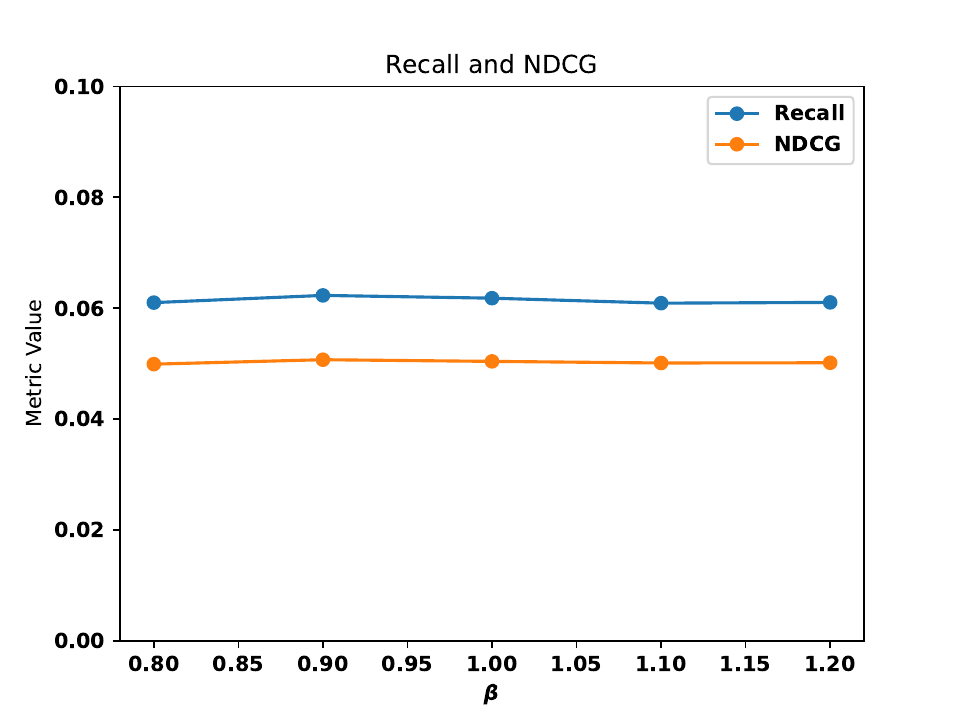}
    \caption{Recall@20 and NDCG@20 on the Yelp2018 dataset with varying $\beta$.}
    \label{fig:hyper}
\end{figure}

\noindent\textbf{Ablation study.} In Section \ref{sec:gcl4gr}, we use the regulation terms in \cite{chung4leveraging} to mitigate the bias introduced by normalization. We study how these terms affect GR models by conducting ablation study. The results are presented in Table \ref{tab:ablation}. We observe that the regulation terms indeed mitigate the bias, since the performance drops dramatically after the regulation terms are removed. This result further verifies the effectiveness of \cite{chung4leveraging}. We also observe that GR model trained with only these regulation terms fails to get a meaningful result. Thus it is safe to conclude that $\mathcal{L}_{\mathrm{COLES}}$ is the main recommendation loss and $\mathcal{L}_{hom}+\mathcal{L}_{het}$ is only for debiasing.

\begin{table}[h]
    \centering
    \caption{The ablation study on the Yelp2018 dataset. LightGCN$_{\mathrm{COLES}}$ is the LightGCN model trained with Equation \ref{eq:gr_coles}. $-$Regulation is trained with only the single-view GCL loss $\mathcal{L}_{\mathrm{COLES}}$. $-$COLES is trained with only the regulation terms $\mathcal{L}_{\mathrm{hom}}+\mathcal{L}_{\mathrm{het}}$.}
    \begin{tabular}{lcc}
    \toprule
         &  Recall@20 & NDCG@20\\
    \midrule
     LightGCN$_{\mathrm{COLES}}$    & 0.0623 & 0.0507 \\
     $-$Regulation & 0.0301 & 0.0249\\
     $-$COLES& 0.0041 & 0.0033\\
    \bottomrule
    \end{tabular}
    
    \label{tab:ablation}
\end{table}

\subsection{Experimental details}

Experiments are conducted on a server with 80 GB NVIDIA A100 GPU. For the node classification task, we use the code repository of SCE\footnote{\url{https://github.com/szzhang17/Sparsest-Cut-Network-Embedding}}, COLES\footnote{\url{https://github.com/allenhaozhu/COLES}} and SP-GCL \footnote{\url{https://github.com/haonan3/SPGCL}}. For the unsupervised GraphSAGE, we use the example code from torch-geometric \footnote{\url{https://github.com/pyg-team/pytorch_geometric/blob/master/examples/graph_sage_unsup.py}}. We also use the repository of GRACE\footnote{\url{https://github.com/CRIPAC-DIG/GRACE}}, GCA\footnote{\url{https://github.com/CRIPAC-DIG/GCA}}, CCA-SSG\footnote{\url{https://github.com/hengruizhang98/CCA-SSG}}, GGD\footnote{\url{https://github.com/zyzisastudyreallyhardguy/Graph-Group-Discrimination}} and PolyGCL\footnote{\url{https://github.com/ChenJY-Count/PolyGCL}}. For the graph recommender task, we use SELFREC\footnote{\url{https://github.com/Coder-Yu/SELFRec}},  a comprehensive framework for self-supervised and graph recommendation.

We adhered closely to the setting employed in prior studies. Specifically, we use 2-layers GCN for all node classification tasks. We use 3-layers models for GR since they preform best in most of datasets. We set the number of hidden dimension to 512 for GCLs and 64 for GRs. We use the Adam \cite{kingma2014adam} optimizer without weight decay, the learning rates for GCLs are tuned from \{5e-4, 1e-3, 5e-3, 1e-2\} and for GRs are fixed as 1e-2. For GCLs, we use linear model in scikit-learn\footnote{\url{https://scikit-learn.org/stable/modules/linear_model.html}} to predict the downstream tasks. For GR, it is conventional to use mini-batch gradient descent, so we adpot this setting and use 2048 as the batch size. We adopt the default choice from previous papers and set $t=2$ and $\beta=0.9$.
\section{Conclusion}

In this paper, we establish a connection between the fields of GCL and GR. With mild assumptions, we theoretically demonstrate the equivalence between GCL and GR in terms of both encoders and loss functions. This finding not only helps explain the phenomenon observed in previous experiments but also inspires new research directions. Through extensive experiments, we show the interchangeability of GCL and GR losses. The effectiveness of off-the-shelf GCL loss in training GR models is particularly astonishing, as prior work typically considered GCL loss to be unsupervised and required joint training with BPR loss in recommendation tasks. We also provide examples of new research directions that arise from our discoveries.

\newpage
\bibliographystyle{ACM-Reference-Format}
\bibliography{sample-base}

\newpage
\appendix

\section{Proofs} \label{app:proofs}

\begin{lemma}
    The trace of the quadratic form of the Laplacian matrix is equal to the smoothness on the graph, i.e.,

    \begin{equation}
        \mathrm{Tr}(\mathbf{E}^T \mathbf{L} \mathbf{E})=\sum_{u=1}^{n_u}\sum_{i\in \mathcal{N}_u}||e_u-e_i||^2.
    \end{equation}
\end{lemma}

\begin{proof}
    \begin{equation}
        \begin{split}
            \mathrm{Tr}(\mathbf{E}^T \mathbf{L} \mathbf{E}) &=\sum_{k=1}^d\sum_{i=1}^n\sum_{j=1}^n e_{i,k} e_{j,k}\mathbf{L}_{i,j}\\
            &=\sum_{k=1}^d(\sum_{i=1}^n e_{i,k}^2\mathbf{L}_{i,i}+\sum_{i\neq j} e_{i,k} e_{j,k}\mathbf{L}_{i,j})\\
            &=\frac{1}{2}\sum_{k=1}^d(\sum_{i=1}^n e_{i,k}^2\mathbf{L}_{i,i}+2\sum_{i\neq j} e_{i,k} e_{j,k}\mathbf{L}_{i,j}+\sum_{j=1}^n e_{j,k}^2\mathbf{L}_{j,j})\\
            &=\frac{1}{2}\sum_{k=1}^d\sum_{i\neq j}(e_{i,k}^2\adj_{i,j}-2e_{i,k} e_{j,k}\adj_{i,j}+e_{j,k}^2\adj_{i,j})\\
            &=\frac{1}{2}\sum_{k=1}^d\sum_{i\neq j}\adj_{i,j}(e_{i,k} - e_{j,k})^2\\
            &= \frac{1}{2}\sum_{i=1}^n\sum_{j=1}^n \adj_{i,j} ||e_i - e_j||^2\\
            &=\sum_{u=1}^{n_u}\sum_{i\in \mathcal{N}_u}||e_u-e_i||^2.
        \end{split}
    \end{equation}
\end{proof}

\begin{theorem}
With the normalized embeddings, we have the following equivalence:
\begin{equation}
    \begin{split}
        \frac{K}{2}\mathcal{L}_{\mathrm{COLES}}^+  -\frac{d_{\mathrm{min}}}{2}\mathcal{L}_{\mathrm{COLES}}^- +d_{\mathrm{min}}Kn_u - mK \leq \mathcal{L}_{\mathrm{BPR}} \\
        \leq \frac{K}{2}\mathcal{L}_{\mathrm{COLES}}^+ - \frac{d_{\mathrm{max}}}{4}\ln(\frac{2e^2}{e^2+1})\mathcal{L}_{\mathrm{COLES}}^- + \mathrm{constant},
    \end{split}
\end{equation}
where $\mathrm{constant}=d_{\mathrm{max}}K n_u\ln(\frac{2e^3+2e}{e^2+1}) - mK$.
\end{theorem}

\begin{proof}
    \begin{equation} \label{eq:bpr_de}
        \begin{split}
            \mathcal{L}_{\mathrm{BPR}}&=-\sum_{u=1}^{n_u}\sum_{i\in \neighbor_u}\sum_{j\in \neighbor_u^-}\ln \left(\frac{1}{1+\exp(\hat{y}_{uj}-\hat{y}_{ui})} \right)\\
            &=-\sum_{u=1}^{n_u}\sum_{i\in \neighbor_u}\sum_{j\in \neighbor_u^-}\ln \left(\frac{\exp(\hat{y}_{ui})}{\exp(\hat{y}_{ui})+\exp(\hat{y}_{uj})} \right)\\
            &=-K\sum_{u=1}^{n_u}\sum_{i\in \neighbor_u} \hat{y}_{ui}+\sum_{u=1}^{n_u}\sum_{i\in \neighbor_u}\sum_{j\in \neighbor_u^-}\ln(e^{\hat{y}_{ui}}+e^{\hat{y}_{uj}})\\
            &=\mathcal{L}_{\mathrm{BPR}}^+ + \mathcal{L}_{\mathrm{BPR}}^-,
        \end{split}
    \end{equation}
    where we denote $\mathcal{L}_{\mathrm{BPR}}^+=-K\sum_{u=1}^{n_u}\sum_{i\in \neighbor_u} \hat{y}_{ui}$ as the positive part of the BPR loss, and $\mathcal{L}_{\mathrm{BPR}}^-=\sum_{u=1}^{n_u}\sum_{i\in \neighbor_u}\sum_{j\in \neighbor_u^-}\ln(e^{\hat{y}_{ui}}+e^{\hat{y}_{uj}})$ as the negative part of the BPR loss. Then, we have:
    \begin{equation}\label{eq:bpr_pos}
        \begin{split}
            \mathcal{L}_{\mathrm{BPR}}^+ &=-K\sum_{u=1}^{n_u}\sum_{i\in \neighbor_u} \hat{y}_{ui}\\
            &= \frac{K}{2}\sum_{u=1}^{n_u}\sum_{i\in \neighbor_u} (||e_u - e_i||^2 - 2)\\
            &= \frac{K}{2}\mathcal{L}_{\mathrm{COLES}}^+ - mK.
        \end{split}
    \end{equation}
    For the negative part, we have:
    \begin{equation}\label{eq:bpr_neg1}
        \begin{split}
            \mathcal{L}_{\mathrm{BPR}}^- &\geq \sum_{u=1}^{n_u}d_u\sum_{j\in \neighbor_u^-}\hat{y}_{uj}\\
            &\geq d_{\mathrm{min}}\sum_{u=1}^{n_u}\sum_{j\in \neighbor_u^-}\hat{y}_{uj}\\
            &=-\frac{d_{\mathrm{min}}}{2}\sum_{u=1}^{n_u}\sum_{j\in \neighbor_u^-} (||e_u - e_j||^2 - 2)\\
            &= -\frac{d_{\mathrm{min}}}{2}\mathcal{L}_{\mathrm{COLES}}^- +d_{\mathrm{min}}Kn_u,
        \end{split}
    \end{equation}
    and,
    \begin{equation}\label{eq:bpr_neg2}
        \begin{split}
            \mathcal{L}_{\mathrm{BPR}}^- &\leq
            d_{\mathrm{max}}\sum_{u=1}^{n_u}\sum_{j\in \neighbor_u^-}\ln(e+e^{\hat{y}_{uj}})\\
            &\leq d_{\mathrm{max}}\sum_{u=1}^{n_u}\sum_{j\in \neighbor_u^-} \left(\frac{1}{2}\ln(\frac{2e^2}{e^2+1})(\hat{y}_{uj}+1)+\ln(e+\frac{1}{e}) \right)\\
            &=\frac{d_{\mathrm{max}}}{2}\ln(\frac{2e^2}{e^2+1})(2K n_u-\frac{\mathcal{L}_{\mathrm{COLES}}^-}{2})+d_{\mathrm{max}}K n_u\ln(e+\frac{1}{e}).
        \end{split}
    \end{equation}
    By plugging equations (\ref{eq:bpr_pos}), (\ref{eq:bpr_neg1}), and (\ref{eq:bpr_neg2}) back into equation (\ref{eq:bpr_de}), we obtain the desired inequality.
\end{proof}

\end{document}